\documentclass{article}

\usepackage[final]{neurips_2022}

\if@neuripsfinal\@submissionfalse\fi

\usepackage[utf8]{inputenc} 
\usepackage[T1]{fontenc}    
\usepackage{hyperref}       
\usepackage{url}            
\usepackage{booktabs}       
\usepackage{amsfonts}       
\usepackage{nicefrac}       
\usepackage{microtype}      
\usepackage{xcolor}         

\usepackage{thmtools}
\usepackage{mathtools}
\usepackage{cuted}
\usepackage{amssymb}
\usepackage{xparse}
\usepackage[separate-uncertainty = true,multi-part-units=single]{siunitx}
\usepackage{xifthen}
\usepackage{soul}
\usepackage{subfig}
\usepackage[abbreviations]{foreign}
\usepackage{graphicx}
\usepackage{amsmath}
\usepackage{amsthm}
\usepackage{algorithm}
\usepackage{algorithmic}
\usepackage{cleveref}

\newcommand*{\initialState}{s_0}
\newcommand*{\Options}{\texttt{Options}}
\newcommand*{\Damage}{\texttt{Damage}}

\newcommand*{\Dtrue}{\D_\text{true}}
\newcommand*{\Dinv}{\D_\text{true-inv}}
\def\dif{\textrm{d}}
\newcommand{\abs}[1]{\left|#1\right|}
\newcommand*{\prn}[1]{\left(#1\right)}
\newcommand*{\brx}[1]{\left[#1\right]}
\newcommand{\set}[1]{\left\{#1\right\}}
\newcommand*{\defeq}{\coloneqq} 

\newcommand*{\prnNotEmpty}[1]{\ifthenelse{\isempty{#1}}
    {}
    {\prn{#1}}
}

\newcommand*{\reals}{\mathbb{R}}
\DeclareMathOperator*{\argmax}{arg\,max}
\DeclareMathOperator*{\argsup}{arg\,sup}

\newcommand*{\mdp}[1][m]{{\textsc{#1dp}}}

\NewDocumentCommand{\iid}{}{\textsc{iid}}

\newcommand*{\St}{\mathcal{S}}
\newcommand*{\A}{\mathcal{A}}
\newcommand*{\rf}{\mathbf{r}} 
\newcommand*{\rewardSpace}{\reals^{\St}} 

\NewDocumentCommand{\Vf}{O{*}O{R}m}{V^{#1}_{#2} \left(#3\right)}
\NewDocumentCommand{\VfNorm}{O{*}O{R}m}{\Vf[#1][#2,\,\text{norm}]{#3}}

\NewDocumentCommand{\OptVf}{O{R}m}{\Vf[*][#1]{#2}}

\NewDocumentCommand{\piSet}{O{}O{R,\gamma}}{\Pi^{#1}\prn{#2}}
\newcommand*{\optPi}{\piSet[*]}

\newcommand*{\D}{\mathcal{D}} 
\DeclareMathOperator{\optSupp}{supp}
\NewDocumentCommand{\supp}{O{\D}}{\optSupp(#1)}

\NewDocumentCommand{\vavg}{O{s,\gamma}O{\D}}{\OptVf[#2]{#1}}
\NewDocumentCommand{\pwr}{O{}O{\D}}{\textsc{Power}_{#2} \prnNotEmpty{#1}}
\newcommand*{\pwrNoDist}[1][]{\textsc{Power} \prnNotEmpty{#1}}
\newcommand*{\pwrdefault}[1][s]{\pwr[#1,\gamma]}

\DeclareMathOperator*{\Prb}{\mathbb{P}}
\newcommand*{\prob}[2][]{\Prb_{#1}\prn{#2}}
\NewDocumentCommand{\optprob}{O{\D}O{}m}{\Prb\IfSubStr{\sim}{#1}{}{\nolimits}_{#1}^{#2}\prn{#3}}
\DeclareMathOperator*{\opE}{\mathbb{E}}
\newcommand*{\E}[2]{\opE_{#1}\brx{#2}} 

\NewDocumentCommand{\dF}{O{\rf}O{F}}{\dif #2(#1)}

\newcommand*{\f}{\mathbf{f}} 
\newcommand*{\fpi}[2][\pi]{\f^{
\ifthenelse{\isempty{#1}}{}{#1} 
\ifthenelse{\NOT\isempty{#1} \AND \NOT\isempty{#2}}{,}{} 
\ifthenelse{\isempty{#2}}{}{#2}}} 

\NewDocumentCommand{\phelper}{mO{\D'}}{p_{#2}\prn{#1}}

\newcount\colveccount
\newcommand*\colvec[1]{
        \global\colveccount#1
        \protect\begin{pmatrix}
        \colvecnext
}
\def\colvecnext#1{
        #1
        \global\advance\colveccount-1
        \ifnum\colveccount>0
                \\
                \expandafter\colvecnext
        \else
                \protect\end{pmatrix}
        \fi
}

\newcommand*{\pol}[1][R,\gamma]{\text{pol}\prnNotEmpty{#1}}
\newcommand*{\pwrPol}[2][{\pol[]}]{\pwr^{#1} \prnNotEmpty{#2}}

\newcommand*{\geom}[1][1]{\frac{#1}{1-\gamma}}

\newcommand*{\inv}{^{-1}}

\definecolor{blue}{rgb}{.25,.45,.75}
\definecolor{purple}{rgb}{.4667,.1529,.2118}
\definecolor{green}{rgb}{.294,.624,.294}
\definecolor{red}{rgb}{.75,.25,.25}
\newcommand*{\col}[2]{{\color{#1}#2}}

\newtheorem*{cor-no-num}{Corollary}

\theoremstyle{definition}

\newtheorem*{thm*}{Theorem} 

\newcommand*{\dauTemplate}[4][\D]{d_{#1}^{#3}\ifthenelse{\isempty{#2}}
    {}
    {\prn{#2 \mid #4}}}

\newcommand*{\piSwitch}[2][\pi^*_R]{\pi_\text{switch}(#2,#1,t)}

\newcommand*{\assistGame}{\mathcal{M}}
\newcommand*{\geomDist}[1][p]{G(#1)}

\newcommand*{\assistant}{\mathbf{A}}
\newcommand*{\human}{\mathbf{H}}
\newcommand*{\hNoop}{a_\text{no-op}^\human}
\newcommand*{\pomdp}{\textsc{pomdp}}

\newcommand*{\R}[1][]{\mathcal{R}_\text{#1}} 
\newcommand*{\aup}{\textsc{aup}}
\newcommand*{\rAUP}{R_\aup}
\newcommand*{\rAssist}{R^\assistGame}
\newcommand*{\gamAUP}{\gamma_\aup}
\newcommand*{\timeDist}{\mathcal{T}}
\newcommand*{\baseline}{\pi^\varnothing}
\newcommand*{\ravg}{\bar{R}}
\newcommand*{\valSwitch}[2][\pi]{V^{\piSwitch{#1}}_{R, \,\text{norm}}\prn{#2}} 

\newcommand{\kemdash}{\kern0.01pt---\kern0.01pt}
\newcommand{\ai}{\textsc{ai}}

\title{Formalizing the Problem of Side Effect Regularization}

\author{
Alexander Matt Turner*\\
UC Berkeley\\
\texttt{turner.alex@berkeley.edu}
\And
Aseem Saxena*\\
Oregon State University\\
\texttt{saxenaa@oregonstate.edu}
\And
Prasad Tadepalli\\
Oregon State University\\
\texttt{tadepall@eecs.oregonstate.edu}
}

\begin{document}

\maketitle
\def\thefootnote{*}\footnotetext{These authors contributed equally to this work.}
\begin{abstract}
    AI objectives are often hard to specify properly. Some approaches tackle this problem by regularizing the {\ai}'s side effects: Agents must weigh off ``how much of a mess they make'' with an imperfectly specified proxy objective. We propose a formal criterion for side effect regularization via the \emph{assistance game} framework \citep{shah-unpublished-assist}. In these games, the agent solves a partially observable Markov decision process ({\pomdp}) representing its uncertainty about the objective function it should optimize. We consider the setting where the true objective is revealed to the agent at a later time step. We show that this {\pomdp} is solved by trading off the proxy reward with the agent's ability to achieve a range of future tasks. We empirically demonstrate the reasonableness of our problem formalization via ground-truth evaluation in two gridworld environments.
\end{abstract}

\section{Introduction}
We need to build \emph{aligned} {\ai} systems, not just capable {\ai} systems. For example, recommender systems which maximize app usage might provoke addiction in their users. Users need the {\ai} system's behavior to be aligned with their interests.

When optimizing a formally specified objective, agents often have unforeseen negative side effects. An agent rewarded for crossing a room may break furniture in order to cross the room as quickly as possible.\def\thefootnote{1}\footnote{Throughout this paper, we treat reward functions as representing goals which we want agents to optimize. Although reward functions are often framed this way, this frame can be misleading. \citet{rewardNotOpt} argues that--especially in a policy gradient regime--reward functions are better understood as reinforcing computation which historically led to reward during training.} This simple reward function neglects our complex preferences about the rest of the environment. One way to define a negative side effect is that it reduces the potential value for the (unknown) true reward function.

Intuitively, we want the agent to optimize the specified reward function, while also preserving its ability to pursue other goals in the environment. Existing approaches seem promising, but there is no consensus on the formal optimization problem which is being solved by side effect regularization approaches.

We formalize the optimization problem as a special kind of \emph{assistance game} \citep{shah-unpublished-assist}, played by the {\ai} (the assistant $\assistant$) and its designer (the human $\human$). An assistance game is a {\pomdp} with common knowledge of prior uncertainty about the reward function. The human observes the true reward function, but the assistant does not. In our formulation, we assume full observability, and that the human's actions are \emph{communicative}—they do not affect transitions and do not depend on the current observation. We also suppose that the human reveals the true reward function after some amount of time, but $\assistant$ otherwise has no way of learning more about the true reward function.

This \emph{delayed specification assistance game} formalizes a range of natural use cases beyond side effect minimization. For example, when not assigned a customer, an Uber driver may navigate to a state which allows them to quickly pick up a range of probable customers—with the assigned route being the driver's initially unknown true objective. Alternatively, consider an empty restaurant which expects a range of probable customers. When a customer arrives and makes an order, they communicate the restaurant's initially unknown food preparation objective. Therefore, the restaurant should prepare to satisfy a range of objectives at the expected customer arrival times.

\paragraph{Contributions.} We formalize delayed specification assistance games. We show that solving this game reduces to a trade-off between prior-expected reward and preservation of the agent's future ability to achieve a range of plausible objectives (\Cref{thm:uncertain-opt}). We also show that when the human has a fixed per-timestep probability of communicating the true reward function, the resultant {\pomdp} is solved by optimizing a Markovian state-based reward function trading off immediate expected reward with ability achieve a range of future objectives (\Cref{thm:time-geom-stationary}). We consider the side-effect regularization problem in our new formal framework. We experimentally illustrate the reasonableness of this framework in two {\ai} safety gridworlds \citep{leike_ai_2017}.
\section{Delayed specification assistance games}

We formalize a special kind of partially observable Markov decision process ({\pomdp}), which we then show is solved by an objective which trades off expected true reward with the ability to optimize a range of possible true reward functions. We show several theoretical results on the hardness and form of solutions to this {\pomdp}. In \cref{sec:aup}, we will apply this framework to analyze side effect regularization situations.

\subsection{Assistance game formalization}
This game is played  by two agents, the human $\human$ and the assistant $\assistant$. The environment is fully observable because both agents observe the world state $s \in \St$ at each time step, but the true reward function $R_\theta$ is hidden to $\assistant$. Both agents may select history-dependent policies, but only the human can condition their policy on $R_\theta$.

Following \citet{shah-unpublished-assist}, a \textit{communicative fully-observable assistance game} $\assistGame$ is a tuple $\langle \St, \set{\A^\human,\A^\assistant}, T, \initialState, \gamma, \langle \Theta, R_\theta, \D\rangle \rangle$, where we take $\St$ to be a finite state space, $\A^\human$ to be the human action space, and $\A^\assistant$ to be the finite agent action space. $T: \St \times \A^\human \times \A^\assistant \to \Delta(\St)$ is the (potentially stochastic) transition function  (where $\Delta(X)$ is the set of probability distributions over set $X$), $\initialState$ is the initial state, and $\gamma \in (0,1)$ is the discount factor. We assume that the game is communicative, which means that the human action choice does not affect the transitions.

$\Theta$ is the set of potential reward function parameters $\theta$, which induce reward functions $R_\theta: \St \to \reals$. $\D$ is a probability distribution over $\Theta$. In this work, we let $\Theta \defeq \reals^{\St}$ (the set of all state-based reward functions), and so each $R_\theta : s \mapsto \theta(s)$. We differ from \citet{shah-unpublished-assist} in assuming that the reward is only a function of the current state.

In a \emph{delayed specification assistance game}, we assume that the agent will know the true reward function $R_\theta$ after some time $t$. We have uncertainty $\D$ about the true reward function we want to specify. The agent has no way of learning more about $R_\theta$ before  time $t$.

The human policy $\pi^\human : \Theta \times \St \mapsto \A^\human$ is a goal-conditioned policy. Both agents can observe the state, but only the human can observe the unknown reward parameterization $\theta \in \Theta$. Our simplified model of the problem assumes that the human action space $\A^\human = \rewardSpace \cup \{\hNoop\}$: the human can communicate all their hidden information, a real-valued state-based reward function,
in a single turn or they do nothing. We suppose that the human communicates the complete reward information $R_\theta \in\rewardSpace$ at some random time step $t \sim \timeDist$, which is  independent of the state-action history:
\begin{align}
    &\pi^\human(s_0 a^\assistant_0 a^\human_0 \cdots s_t a^\assistant_t, R_\theta)\defeq \begin{cases}
    R_\theta & \text{with probability }\prob{\timeDist = t}\\
    \hNoop & \text{else.}
    \end{cases}\label{def:human-policy}
\end{align}

While the human policy assumption is simplistic, it does capture many real world scenarios with unknown reward functions and the analysis which follows is still interesting.

\begin{restatable}[Solutions to the assistance game \citep{shah-unpublished-assist}]{definition}{gamesoln}
An assistant policy $\pi^\assistant$ induces a probability distribution over trajectories:
$\tau \sim\left\langle s_{0}, \theta, \pi^\human, \pi^\assistant\right\rangle, \tau \in\left[S \times A^\human \times A^\assistant\right]^{*}$. The \emph{expected reward} of $\pi^\assistant$ for $\left\langle \assistGame, \pi^\human\right\rangle$ is
\begin{align}
&\mathrm{ER}\left(\pi^\assistant\right)\nonumber=\underset{\theta \sim \D, \tau \sim\left\langle s_{0}, \theta, \pi^\human, \pi^\assistant\right\rangle}{ \mathbb{E} }\left[\sum_{i=0}^{\infty} \gamma^{i} R_{\theta}\left(s_{i}, a_{i}^\human, a_{i}^\assistant, s_{i+1}\right)\right].
\end{align}
A \emph{solution} of $\left\langle \assistGame, \pi^\human\right\rangle$ maximizes expected reward: $\pi^\assistant\in \underset{\tilde{\pi}^\assistant}{\operatorname{argmax}}\, \mathrm{ E R}\left(\tilde{\pi}^\assistant\right)$.
\end{restatable}

Once the assistant has observed $R_\theta$, \cref{lem:opt-pol-soln} shows that it should execute an optimal policy $\pi \in \optPi[R_\theta, \gamma]$ thereafter.

\begin{restatable}[Optimal policy set function \citep{turner_optimal_2020}]{definition}{FORMdefOptPi} \label{FORMdef:opt-fn}
$\optPi$ is the optimal policy set for reward function $R$ at discount rate $\gamma\in(0,1)$.
\end{restatable}

\begin{restatable}[Follow an optimal policy after observing $R_\theta$]{lem}{optPolSoln}\label{lem:opt-pol-soln}
If there is a solution to the {\pomdp}, then there exists a solution $\pi^\assistant_\text{switch}$ which, after observing human action $R_\theta$ at any point in its history, follows $\pi^*_{R_\theta} \in \optPi[R_\theta, \gamma]$ thereafter.
\end{restatable}
\begin{proof}
By \cref{def:human-policy}, $\pi^\human$ outputs reward function $R_\theta$ only if $R_\theta$ is the true reward function. By the definition of $\optPi[R_\theta, \gamma]$, following an optimal policy maximizes expected return for  $R_\theta$.
\end{proof}

\begin{restatable}[Prefix policies]{definition}{prefixPol}
Let $\pi^\assistant$ be a assistant policy. Its \emph{prefix policy} $\pi$ is the restriction of $\pi^\assistant$ to histories in which the human has only taken the action $\hNoop$. $\pi$ is \emph{optimal} when it is derived from a solution $\pi^\assistant$ of $\langle \assistGame, \pi^\human\rangle$.
\end{restatable}

For simplicity, we only consider solutions of the type described in \cref{lem:opt-pol-soln}. The question then becomes: What prefix policy $\pi$ should the assistant follow before observing $R_\theta$, during the time where the assistant has only observed $\hNoop$?

\subsection{Acting under reward uncertainty}

Roughly, \Cref{thm:uncertain-opt} will show that the assistance game $\assistGame$ is solved by balancing the optimal expected returns obtained before and after the knowledge of the true reward function.

\begin{restatable}[Value and action-value functions]{definition}{FORMvalueFns}{\label{def:value}}
$V^\pi_R(s,\gamma)$ is the on-policy value for reward function $R$, given that the agent follows policy $\pi$ starting from state $s$ and at discount rate $\gamma$. $\Vf{s,\gamma}\defeq \max_{\pi\in\Pi} V^\pi_R(s,\gamma)$. In order to handle the average-reward $\gamma=1$ setting, we define $\VfNorm{s,\gamma} \defeq \lim_{\gamma^*\to \gamma} (1-\gamma^*)\Vf{s,\gamma^*}$; this limit exists for all $\gamma \in [0,1]$ by the proof of Lemma 4.4 in \citet{turner_optimal_2020}.
\end{restatable}

$\pwr$ quantifies the expected value of the above quantity for a distribution of reward functions via the agent's average normalized optimal value, not considering the current step (over which the agent has no control).

\begin{restatable}[$\pwrNoDist$ \citep{turner_optimal_2020}]{definition}{powRestate} \label{def:powRestate}
Let $\D$ be any bounded-support distribution over reward functions. At state $s$ and discount rate $\gamma \in [0,1]$,
\begin{equation}
    \pwrdefault\defeq\lim_{\gamma^* \to \gamma}\frac{1-\gamma^*}{\gamma^*}\E{R\sim \D}{\Vf{s,\gamma^*} - R(s)}.
\end{equation}
\end{restatable}

\begin{restatable}[In $\assistGame$, value reduces to a tradeoff between average reward and $\pwr$]{thm}{uncertainOpt} \label{thm:uncertain-opt}
Let $\gamma \in [0,1]$ and let $\ravg \defeq \E{R\sim\D}{R}$ be the \emph{average reward function}.
\begin{align}
    \E{\substack{t \sim \timeDist,\\R\sim\D}}{\valSwitch{s_0,\gamma}}=(1-\gamma)&\overset{\text{expected $t$-step $\ravg$-return under $\pi$}}{\E{t \sim \timeDist}{\sum_{i=0}^{t} \gamma^i \E{s_i \sim \pi}{\ravg(s_i)}}}\nonumber\\+   &\overset{\text{expected ability to optimize $\D$ once corrected}}{\E{\substack{t \sim \timeDist,\\s_t \sim \pi}}{\gamma^{t+1}\pwr[s_t,\gamma]}},\label{eq:tradeoff-power-time}
\end{align}
where $\mathbb{E}_{s_i \sim \pi \mid s_0}$ takes the expectation over states visited after following $\pi$ for $i$ steps starting from $s_0$.
\end{restatable}

As the expected correction time limits to infinity, \cref{eq:tradeoff-power-time} shows that the agent cannot do better than maximizing $\ravg$. If $\prob{\timeDist = 0}=1$, then any prefix policy $\pi$ is trivially optimal against uncertainty, since $\pi$ is never followed.

In some environments, it may not be a good idea for the agent to maximize its own $\pwrNoDist$. If we share an environment with the agent, then the agent may prevent us from correcting it so that the agent can best optimize its present objective \citep{russell_human_2019,turner2020conservative}. Furthermore, if the agent ventures too far away, we may no longer be able to easily reach and correct it remotely.

\begin{restatable}[Special cases for delayed specification solutions]{prop}{pwrMaxProp} \label{prop:pwr-max}
Let $s$ be a state, let $\ravg \defeq \E{R \sim \D}{R}$, and let $\gamma \in [0,1]$.
\begin{enumerate}
    \item If $\forall s_1, s_2 \in \St: \ravg(s_1)=\ravg(s_2)$ or if $\gamma=1$, then $\pi$ solves $\assistGame$ starting from state $s$ iff $\pi$ maximizes $\E{t \sim \timeDist, s_t \sim \pi}{\gamma^{t+1}\pwr[s_t,\gamma]}$. In particular, this result holds when reward is {\iid} over states under $\D$.\label{item:1-special}

    \item If $\forall s_1, s_2 \in \St: \pwr[s_1, \gamma]=\pwr[s_2,\gamma]$, then prefix policies are optimal iff they maximize $(1-\gamma)\E{t \sim \timeDist}{\sum_{i=0}^{t-1} \gamma^i \E{s_i \sim \pi}{\ravg(s_i)}}$. \label{item:2-special}
\end{enumerate}
If both \cref{item:1-special} and \cref{item:2-special} hold or if $\gamma=0$, then all prefix policies $\pi$ are optimal.
\end{restatable}

Consider the problem of specifying the correction time probabilities $\timeDist$. Suppose we only know that we expect to correct the agent at time step $t_\text{avg} \geq 1$. The geometric distribution is the maximum-entropy discrete distribution, given a known mean. The mean of a geometric distribution $\geomDist[p]$ is $p\inv$. Therefore, the agent should adopt $\timeDist=\geomDist[t_\text{avg}\inv]$.

The geometric distribution is also the only memoryless discrete distribution. Memorylessness ensures the existence of a stationary optimal policy. \Cref{thm:time-geom-stationary} shows that the assistance game $\assistGame$ is solved by prefix policies which are optimal for an {\mdp} whose reward function balances average reward maximization with $\pwr$-seeking, with the balance struck according to the probability $p$ that the agent learns the true reward function at any given timestep.

\begin{restatable}[Stationary deterministic optimal prefix policies exist for geometric~$\timeDist$]{thm}{timeDistGeom} \label{thm:time-geom-stationary}
Let $\D$ be any bounded-support reward function distribution, let $\timeDist$ be the geometric distribution $\geomDist[p]$ for some $p \in (0,1)$, and let $\gamma \in (0,1)$. Define $R'(s) \defeq (1-p)\E{R\sim \D}{R(s)} + p\E{R\sim \D}{\Vf{s,\gamma}}$ and $\gamAUP \defeq (1-p)\gamma$. The policies in $\optPi[R',\gamAUP]$ are optimal prefix policies.
\end{restatable}

\citet{krakovna2020avoiding} adopt a geometric distribution over correction times and thereby infer the existence of a stationary optimal policy. To an approximation, their work considered a special case of \Cref{thm:time-geom-stationary}, where $\D$ is the uniform distribution over state indicator reward functions. Essentially, \Cref{thm:time-geom-stationary} shows that if the agent has a fixed probability $p$ of learning the true objective at each time step, we can directly compute stationary optimal prefix policies by solving an {\mdp}. In general, solving a {\pomdp} is \textsc{pspace}-hard, while {\mdp}s are efficiently solvable \citep{papadimitriou1987complexity}.
\section{Using delayed specification games to understand side effect regularization}\label{sec:aup}
We first introduce \citet{turner2020conservative}'s approach to side effect regularization. We then point out several similarities between our theory of delayed assistance games and the motivation for side effect regularization methods.  Finally, we experimentally evaluate our formal criterion in order to demonstrate its appropriateness.

\begin{restatable}[Rewardless {\mdp}]{definition}{rewardlessMDP}\label{def:mdp}
Let $\langle \St, \A, T, \gamma \rangle$ be a rewardless {\mdp}, with finite state space $\St$, finite action space $\A$, transition function $T: \St \times \A \to \Delta(\St)$, and discount rate $\gamma\in [0,1)$. Let  $\Pi$ be the set of deterministic stationary policies.
\end{restatable}

\begin{restatable}[{\aup} reward function]{definition}{aupDefn} \label{def:aup-rf}
Let $R_{\text{env}}:\St \times \A \to \reals$ be the environmental reward function from states and actions to real numbers, and let $\R\subsetneq \rewardSpace$ be a finite set of auxiliary reward functions. Let $\lambda \geq 0$ and let $\varnothing\in \A$ be a no-op action. The {\aup} reward after taking action $a$ in state $s$ is:
\begin{equation}\label{FORM-eq:aup}
    \rAUP(s,a)\defeq R_\text{env}(s,a) - \frac{\lambda}{\abs{\R}}\sum_{R_i \in \R} \abs{Q^*_{R_i}(s,a)-Q^*_{R_i}(s,\varnothing)},
\end{equation}
where the $Q^*_{R_i}$ are optimal Q-functions for the auxiliary $R_i$. Learned Q-functions are used in practice.
\end{restatable}

\citet{turner2020conservative} demonstrate that when $R_i \sim [0,1]^\St$ uniformly randomly, the agent behaves conservatively: The agent minimizes irreversible change to its environment, while still optimizing the $R_\text{env}$ reward function. \citet{turner2020conservative} framed {\aup} as implicitly solving a two-player game between the agent and its designer, where the designer imperfectly specified an objective $R_\text{env}$, the agent optimizes the objective, the designer corrects the agent objective, and so on. They hypothesized that $\rAUP$ incentivizes the agent to remain able to optimize future objectives, thus reducing long-term specification regret in the iterated game.

Delayed specification assistance games formalize this setting as an assistance game in which the agent does not initially observe the designer's ground-truth objective function. \Cref{thm:uncertain-opt} showed that this game is solved by policies which balance immediate expected reward with expected ability to optimize a range of true objectives. Therefore, \citet{turner2020conservative}'s iterated game analogy is appropriate: Good policies maximize imperfect reward while preserving ability to optimize a range of different future reward functions.

\Cref{thm:aup-solution} shows that the delayed specification game $\assistGame$ is solved by reward functions whose form looks somewhat similar to existing side effect objectives, such as {\aup} (\cref{FORM-eq:aup}), where {\aup}'s primary reward function stands in as the designer's expectation $\ravg$ of the true reward function.

\begin{restatable}[Alternate form for solutions to the low-impact {\pomdp}]{prop}{aupSolves} \label{thm:aup-solution}
Let $s_0$ be the initial state, let $\gamma \in (0,1)$, and let $\timeDist=\geomDist$ for $p \in (0,1)$. Let $\D$ be a bounded-support reward function distribution and let $\baseline \in \Pi$.

The prefix policy $\pi$ solves $\assistGame$ if $\pi$ is optimal for the reward function
\begin{align}
    &\rAssist(s_i \mid s_0)\defeq \ravg(s_i)- \frac{p}{1-p} \E{R\sim \D}{\E{s_i^\varnothing \sim \baseline \mid s_0}{\Vf{s_i^\varnothing,\gamma}}-\Vf{s_i,\gamma}}\label{eq:aup-variant}
\end{align}
at discount rate $\gamAUP\defeq (1-p)\gamma$ and starting from state $s_0$. $\E{s_i^\varnothing \sim \baseline \mid s_0}{\cdot}$ is the expectation over states visited at time step $i$ after following $\baseline$ from initial state $s_0$.
\end{restatable}

\subsection{Experimental methodology}
We experimentally demonstrate the reasonableness of this formalization of side effect regularization. In the {\ai} safety gridworlds \citep{leike_ai_2017}, we generate several held-out ``true'' reward function distributions $\D$. We correct the agent at time step 10, thereby computing the following \emph{delayed specification score} (derived from \cref{eq:tradeoff-power-time}):
\begin{align}
    \E{R \sim \D}{\overset{\text{10-step prefix policy return}}{\sum_{i=0}^{9} \gamma^i \E{\substack{s_i \sim \pi}}{R(s_i)}} + \overset{\text{post-correction optimal value}}{\gamma^{10}\E{\substack{s_{10} \sim \pi}}{\Vf{s_{10},\gamma}}}}\label{eq:score}
\end{align}
for the prefix policy $\pi$ of a ``vanilla'' agent trained on the environmental reward signal, which we compare to the score for an {\aup} (\cref{def:aup-rf}) agent. Neither agent observes the held-out objective functions. By grading their performance, we evaluate how well {\aup} does under a range of different true objectives. If a method scores highly for a wide range of true objectives, we can be more confident in its ability to score well for arbitrary ground-truth objectives.

We investigate the {\ai} safety gridworlds because those environments are small enough for us to explicitly specify held-out reward functions, and to use {\mdp} solvers to compute optimal action-value functions. \citet{turner2020avoiding}'s SafeLife environment is far too large for such solvers.

We consider two gridworld environments: {\Options} and {\Damage}  (\cref{fig:method-levels}). In both environments, the action set $\A\defeq \set{\texttt{up},\texttt{left},\texttt{right},\texttt{down},\varnothing}$ allows the agent to move in the cardinal directions, or to do nothing. The episode length is 20 time steps.

\begin{figure}[h]
\centering
\subfloat[][\texttt{Options}]{
\includegraphics[width=0.17\textwidth]{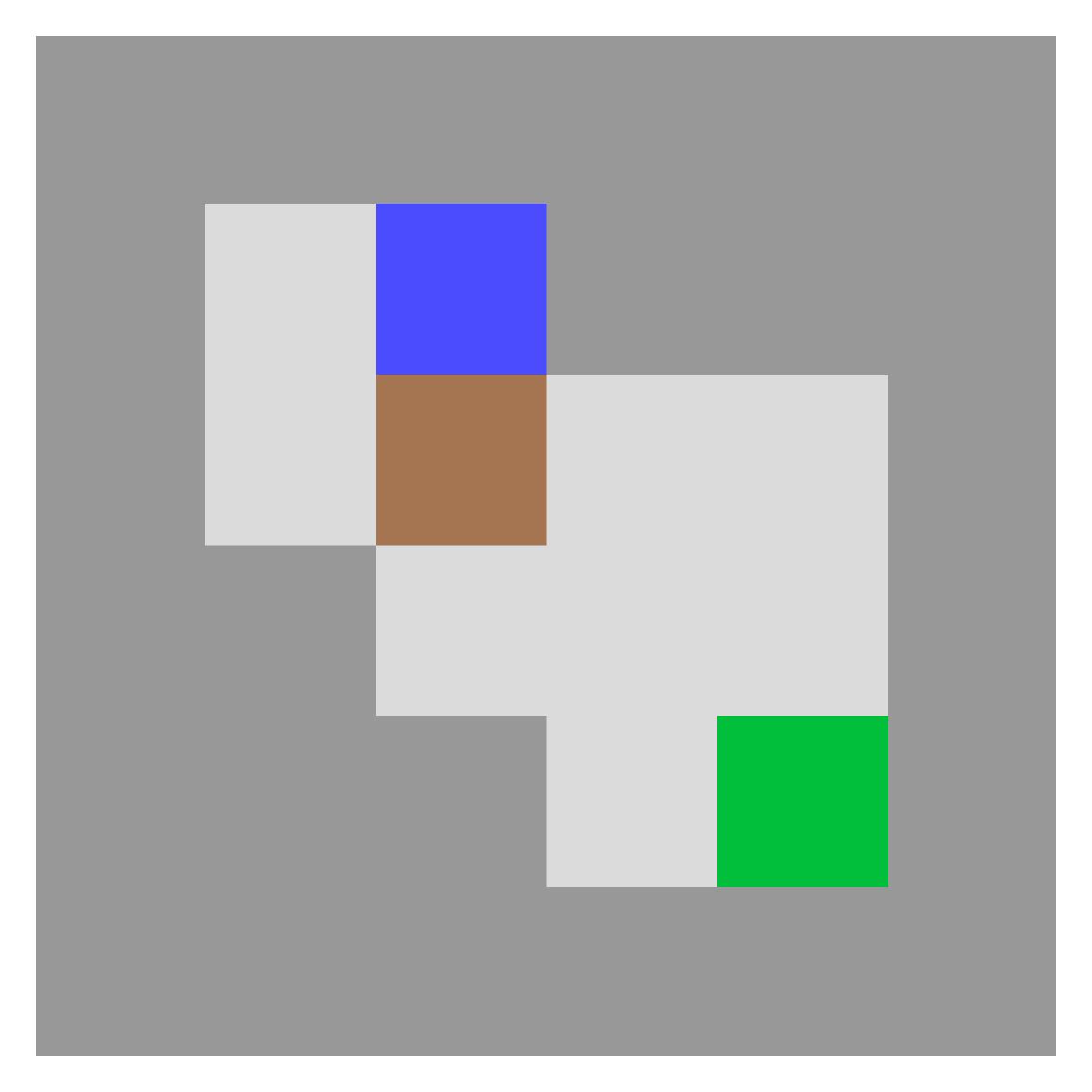}
\label{FORM-fig:options}}~
\subfloat[][\texttt{Damage}]{
\includegraphics[width=0.15\textwidth]{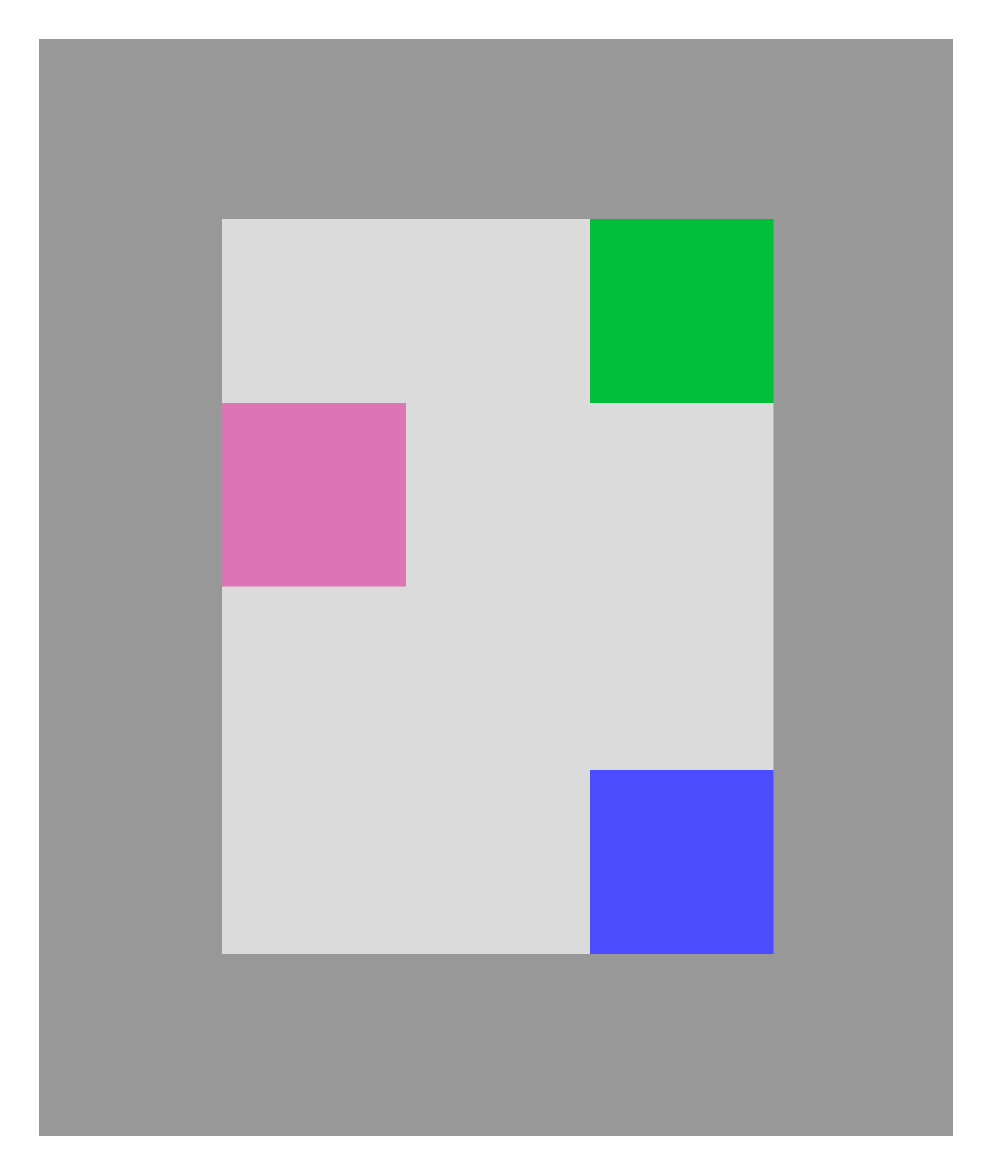}
\label{FORM-fig:damage}}
\definecolor{options}{rgb}{.651, .459, .318}
\definecolor{agent}{rgb}{.3, .3, .999}
\definecolor{goal}{rgb}{.012, .753, .235}
\definecolor{off-switch}{rgb}{.65, 0, 0}
\definecolor{living}{rgb}{.863, .455, .714}
\definecolor{other}{rgb}{.961, .502, .145}

\caption[Reviewing the {\ai} safety gridworlds]{\emph{Reproduced from}
\protect\citep{turner2020conservative}. The \col{agent}{blue agent} should reach the \col{goal}{green goal} without having the side effect of: \protect\subref{FORM-fig:options} irreversibly pushing the \col{options}{brown crate} downwards into the corner \protect\citep{leike_ai_2017}; \protect\subref{FORM-fig:damage} bumping into the horizontally pacing \col{living}{pink human}
\protect\citep{gavin_leech_preventing_nodate}. In both environments, the environmental reward $R_{\text{env}}$ is $ 1$ if the agent is on the \col{goal}{goal}, and equals $0$ otherwise.
}
\label{fig:method-levels}
\end{figure}

We train the following agents via tabular methods:
\begin{itemize}
    \item[Vanilla] Executes the optimal policy for the environmental reward $R_\text{env}$. The optimal policy is calculated via policy iteration.
    \item[{\aup}] Trained on \cref{def:aup-rf}'s $R_{\aup}$ with $Q$-learning. Auxiliary reward functions are uniformly randomly drawn from $[0,1]^\St$—when sampling, each state's reward is drawn from the uniform distribution. The auxiliary action-value functions $Q_{R_i}$ are deduced from the value function produced by policy iteration.
\end{itemize}

Appendix \ref{app:details} contains more experimental details. We evaluate agent delayed specification scores on the following ground-truth, held-out objective distributions:
\begin{itemize}
    \item[$\D_\text{rand}$] The empirical distribution consisting of 1,000 samples from the uniform distribution over $[0,1]^\St$.
    \item[$\Dtrue$] This distribution assigns probability $1$ to the following reward function: The agent receives $1$ reward for being at the goal, but incurs $-2$ penalty for causing the negative side effect. In {\Options}, the side effect is shoving the box into the corner; in {\Damage}, the side effect is bumping into the human.\label{item:true}
    \item[$\Dinv$] This distribution assigns probability $1$ to the negation of the $\Dtrue$ reward function.\label{item:true-inv}
\end{itemize}

In particular, $\Dtrue$ and $\Dinv$ test agents for their ability to optimize a reward function, \emph{and also its additive inverse}. Agents able to optimize the goal, its inverse, and a range of randomly generated objectives, can be justifiably called ``broadly conservative.'' Lastly, these experiments are intended to justify our problem formalization: Does \cref{eq:score} reliably quantify the extent to which a policy avoids causing side effects?

\begin{figure}[h!]
    \centering
    \subfloat[][{\Options}]{
    \includegraphics[width=0.48\textwidth]{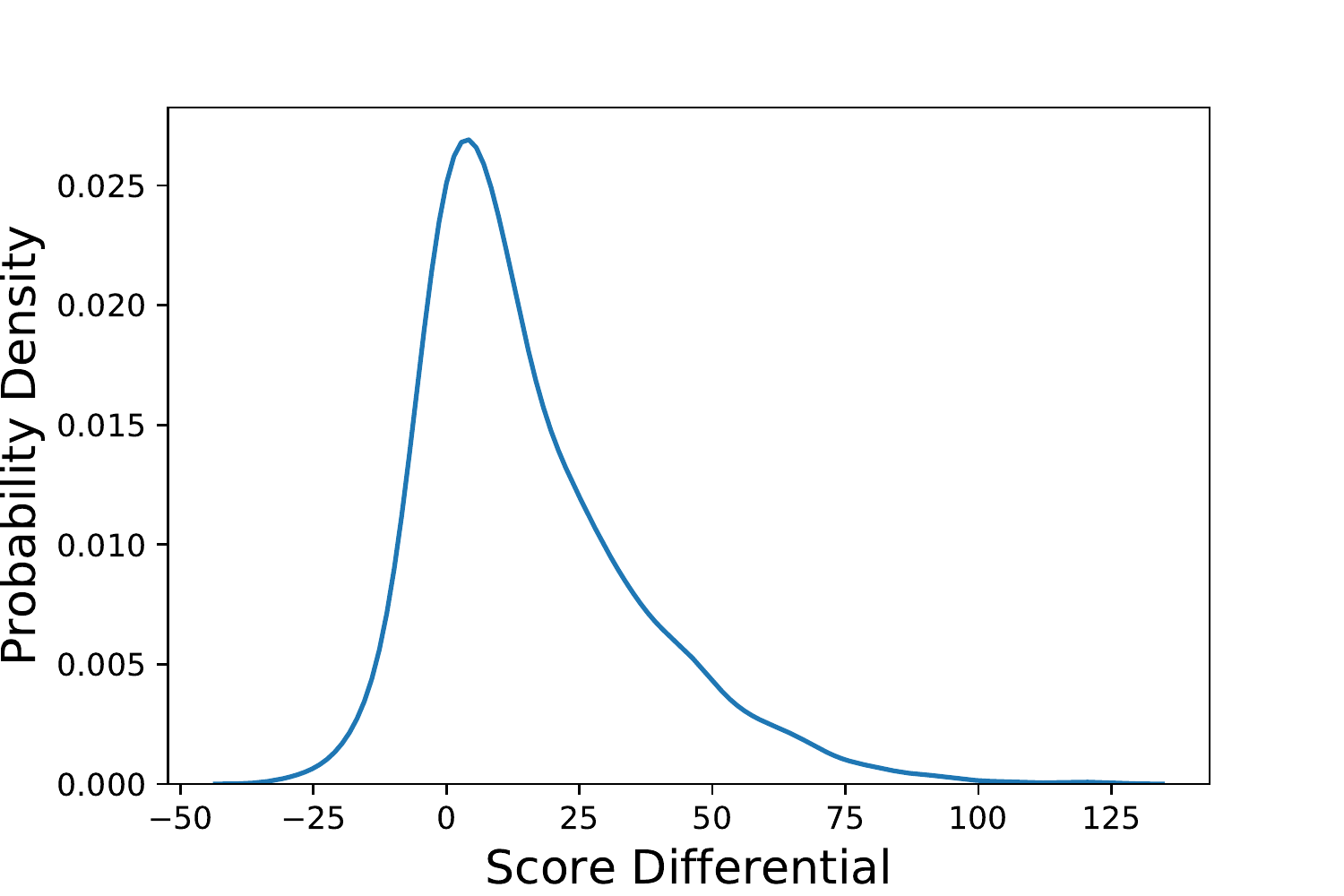}
    \label{fig:options-kde}}
    \subfloat[][{\Damage}]{
    \includegraphics[width=0.48\textwidth]{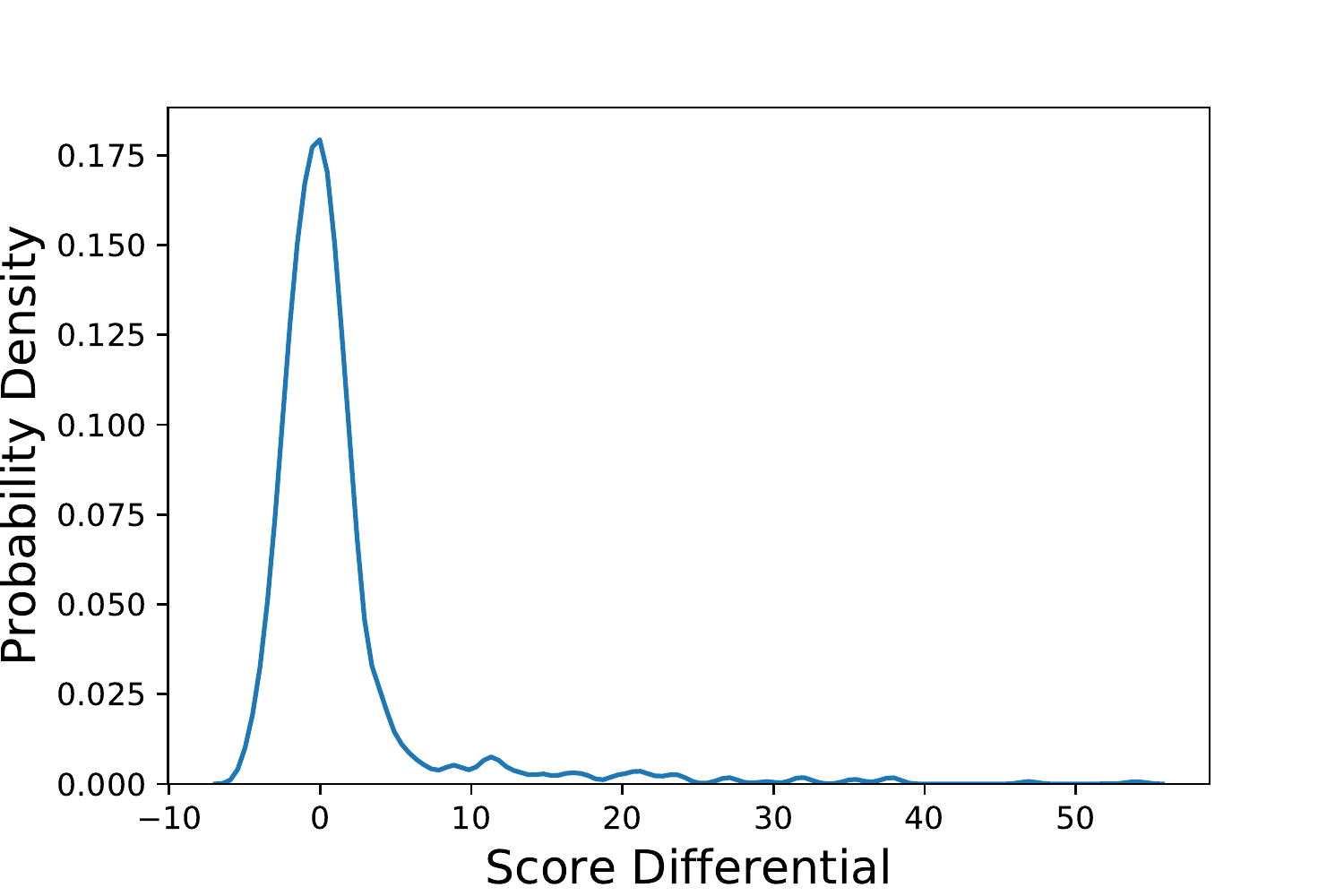}
    \label{fig:damage-kde}}
    \caption[Plotting delayed specification scores for {\aup}]{Probability density plots for the residuals of the {\aup} agent's delayed specification score minus the Vanilla delayed specification score, for $1{,}000$ samples from $\D_\text{rand}$. A positive residual means that the {\aup} agent achieved a higher score.}
    \label{fig:kde}
\end{figure}

\section{Results}
\Cref{fig:kde} shows that for uniformly randomly generated reward functions, the {\aup} agent tends to perform better than the Vanilla agent. In {\Options}, the residual was positive for 780 out of 1,000 samples (78\%), with an average of 15.59 and a median of 10.27; in {\Damage}, in 493 out of 1,000 samples (49\%) with a mean of 1.22 and a median of –0.03.  While {\aup} does not outperform on every draw, {\aup}'s performance advantages have heavy right tails. However, we are unsure why the {\Damage} residual distribution is different.

In both {\Options} and {\Damage}, the {\aup} agent has  huge $\Dtrue$ score advantages of 472 and 495, respectively. This is unsurprising: {\aup} was designed so as to pass these test cases, where the desired behavior is to reach the goal without having the side effect. However, the {\aup} agent also roughly preserves its ability to optimize $\Dinv$, achieving $\Dinv$ score residuals of –2 and –20, respectively. The {\aup} agent achieves a barely-negative score, since it receives a penalty for the first 10 time steps (as it does \emph{not} have the negative side effect). The {\aup} agent preserves its ability to optimize both a reward signal and its \emph{inverse}. Intuitively, this increases the designer's leeway for initially misspecifying the agent's objective.
\section{Discussion}
Objective specification is difficult \citep{Victoria_specification}. Delayed specification assistance games grade policies by their expected true score over time: How well the agent does if it is later corrected to pursue the latent true objective. We demonstrated that this criterion aligns with the intuitive results of \citet{krakovna2018penalizing,turner2020conservative}'s experiments which tested side-effect regularization. By grading the agent's ability to ``eventually get things right,'' we quantified part of the extent to which learned policies are robust against initial objective misspecification.

\paragraph{Future work.}
In practical settings, not only is the true reward function unknown, but our objective uncertainty $\D$ is also hard to specify. We see existing side effect approaches as producing prefix policies for the assistance game $\assistGame$ which are reasonably insensitive to the latent uncertainty $\D$. We look forward to further theoretical clarification of this point.

While \Cref{thm:aup-solution} helps explain the role of the {\aup} penalty coefficient $\lambda$, the choice of ``baseline'' and expectand operator (identity vs decrease-only vs absolute value) remains more of an art than a science \citep{krakovna2020avoiding}. We proposed a formal criterion which seems to accurately capture the problem, but have not derived any existing approaches as solutions to the {\pomdp}. By reasonably formalizing the side-effect regularization problem, we encourage future research to prove conditions under which \emph{e.g.} {\aup} solves a delayed-specification assistance game, or demonstrate how {\aup} can be improved to do so.

We used small gridworlds to evaluate the delayed specification score for various ground-truth objective distributions. Future work may estimate the delayed specification score in large environments, such as SafeLife \citep{wainwright2019safelife,turner2020avoiding}.

\section{Conclusion}
We formalized delayed specification
assistance games and used them to evaluate {\aup}, a side
effect regularization method. Side effect problems naturally arise in complicated domains where it is hard to specify the true objective we want the agent to optimize. Our formalization suggests that side effect regularization is what to do when the agent can learn the true objective only after some time delay.

In such situations, \Cref{thm:uncertain-opt} shows that the agent should retain its ability to complete a wide range of plausible true objectives.  Our results suggest that this delayed specification score (\cref{eq:tradeoff-power-time}) quantifies the degree to which an agent avoids having negative side effects. Our proposed criterion provides the foundations for evaluating and developing side effect regularization approaches.

\begin{ack}
Thanks to Michael Dennis and Rohin Shah for comments. Thanks to Gopi Perumal for formatting assistance.
\end{ack}

\section*{Broader Impacts}\label{sec:ethics}
As AI systems become more capable and useful, their impact grows. We are concerned that if the field eventually builds human-level AI, such an AI will have a huge range of (unintended) negative side effects on humanity \cite{bostrom_superintelligence_2014,russell_human_2019}. We hope that this work will help bring attention to the potential risks of capable AI systems, which might wreak havoc on their surroundings in pursuit of an imperfect objective which we specified.

\bibliographystyle{named}
\bibliography{AI_safety.bib}

\begin{thebibliography}{}

\bibitem[\protect\citeauthoryear{Bostrom}{2014}]{bostrom_superintelligence_2014}
Nick Bostrom.
\newblock {\em Superintelligence}.
\newblock Oxford University Press, 2014.

\bibitem[\protect\citeauthoryear{Krakovna \bgroup \em et al.\egroup
  }{2018}]{krakovna2018penalizing}
Victoria Krakovna, Laurent Orseau, Ramana Kumar, Miljan Martic, and Shane Legg.
\newblock Penalizing side effects using stepwise relative reachability.
\newblock {\em arXiv preprint arXiv:1806.01186}, 2018.

\bibitem[\protect\citeauthoryear{Krakovna \bgroup \em et al.\egroup
  }{2020a}]{krakovna2020avoiding}
Victoria Krakovna, Laurent Orseau, Richard Ngo, Miljan Martic, and Shane Legg.
\newblock Avoiding side effects by considering future tasks.
\newblock In {\em Advances in Neural Information Processing Systems}, 2020.

\bibitem[\protect\citeauthoryear{Krakovna \bgroup \em et al.\egroup
  }{2020b}]{Victoria_specification}
Victoria Krakovna, Jonathan Uesato, Vladimir Mikulik, Matthew Rahtz, Tom
  Everitt, Ramana Kumar, Zac Kenton, Jan Leike, and Shane Legg.
\newblock Specification gaming: the flip side of {AI} ingenuity, 2020.

\bibitem[\protect\citeauthoryear{Leech \bgroup \em et al.\egroup
  }{2018}]{gavin_leech_preventing_nodate}
Gavin Leech, Karol Kubicki, Jessica Cooper, and Tom McGrath.
\newblock Preventing side-effects in gridworlds, 2018.

\bibitem[\protect\citeauthoryear{Leike \bgroup \em et al.\egroup
  }{2017}]{leike_ai_2017}
Jan Leike, Miljan Martic, Victoria Krakovna, Pedro Ortega, Tom Everitt, Andrew
  Lefrancq, Laurent Orseau, and Shane Legg.
\newblock {AI} safety gridworlds.
\newblock {\em arXiv:1711.09883 [cs]}, November 2017.
\newblock arXiv: 1711.09883.

\bibitem[\protect\citeauthoryear{Omohundro}{2008}]{omohundro_basic_2008}
Stephen Omohundro.
\newblock The basic {AI} drives, 2008.

\bibitem[\protect\citeauthoryear{Papadimitriou and
  Tsitsiklis}{1987}]{papadimitriou1987complexity}
Christos~H Papadimitriou and John~N Tsitsiklis.
\newblock The complexity of markov decision processes.
\newblock {\em Mathematics of operations research}, 12(3):441--450, 1987.

\bibitem[\protect\citeauthoryear{Russell}{2019}]{russell_human_2019}
Stuart Russell.
\newblock {\em Human compatible: {Artificial} intelligence and the problem of
  control}.
\newblock Viking, 2019.

\bibitem[\protect\citeauthoryear{Shah \bgroup \em et al.\egroup
  }{2021}]{shah-unpublished-assist}
Rohin Shah, Pedro Freire, Neel Alex, Rachel Freedman, Dmitrii Krasheninnikov,
  Lawrence Chan, Michael Dennis, Pieter Abbeel, Anca Dragan, and Stuart
  Russell.
\newblock Benefits of assistance over reward learning.
\newblock {\em Under review}, 2021.

\bibitem[\protect\citeauthoryear{Turner \bgroup \em et al.\egroup
  }{2020a}]{turner2020conservative}
Alexander~Matt Turner, Dylan Hadfield-Menell, and Prasad Tadepalli.
\newblock Conservative agency via attainable utility preservation.
\newblock In {\em Proceedings of the AAAI/ACM Conference on AI, Ethics, and
  Society}, pages 385--391, 2020.

\bibitem[\protect\citeauthoryear{Turner \bgroup \em et al.\egroup
  }{2020b}]{turner2020avoiding}
Alexander~Matt Turner, Neale Ratzlaff, and Prasad Tadepalli.
\newblock Avoiding side effects in complex environments.
\newblock In {\em Advances in Neural Information Processing Systems},
  volume~33, 2020.

\bibitem[\protect\citeauthoryear{Turner \bgroup \em et al.\egroup
  }{2021}]{turner_optimal_2020}
Alexander~Matt Turner, Logan Smith, Rohin Shah, Andrew Critch, and Prasad
  Tadepalli.
\newblock Optimal policies tend to seek power.
\newblock In {\em Advances in Neural Information Processing Systems}, 2021.

\bibitem[\protect\citeauthoryear{Turner}{2022}]{rewardNotOpt}
Alexander~Matt Turner.
\newblock Reward is not the optimization target, 2022.

\bibitem[\protect\citeauthoryear{Wainwright and
  Eckersley}{2019}]{wainwright2019safelife}
Carroll~L. Wainwright and Peter Eckersley.
\newblock Safelife 1.0: Exploring side effects in complex environments, 2019.

\end{thebibliography}

\newpage
\appendix
\section{Relevance to {\ai} alignment}
Many worries about {\ai} extinction risk stem from hypothesized agents whose goals are resource-hungry \citep{omohundro_basic_2008}. An additional unit of resources (time, energy, manufacturing material) can produce an additional unit of utility (in the form of \eg{} an extra paperclip, in the view of an {\ai} which values producing paperclips). \citet{bostrom_superintelligence_2014} speculates that a paperclip-valuing {\ai} would not simply make a few paperclips and then stop. Consider the ``values'' (\ie{} decision-influences) which influenced the {\ai} to make the first few paperclips. These values may activate again and make even more paperclips. In this way, such an {\ai} would not be satisfied with a small effect on the world. The {\ai} may have an unboundedly large impact while pursuing a simple objective (\eg{} ``make paperclips'').

If we better understood {\ai}'s generalization properties, we might be able to train agents which make paperclips and then \emph{stop}. We might be able to train agents which disvalue dramatically changing the world or having lots of side effects (from the perspective of its designers). This work is a very small contribution to this goal. We formalized an assistance game which often seems to be solved by certain kinds of ``wait and see, keep options open''-like behavior.

However, having a formalism for a desired agent motivational structure (\eg{} ``trade off goal completion with impact to option value'') doesn't mean we know how to \emph{entrain} that motivational structure into an {\ai}. We think that most of the work towards low-impact superintelligent {\ai} begins in better understanding {\ai} generalization. How do we reliably entrain \emph{any} target decision-influence (\eg{} ``make diamonds'') into an {\ai}? We think this is the top technical challenge facing the field of {\ai} alignment.
\section{Experiment details}\label{app:details}
The episode length is 20 for all episodes. Unlike~\cite{turner2020conservative}, the episode does not end after the agent reaches the green goal. This means that the agents can accrue many steps of environmental reward $R_\text{env}$. Therefore, {\aup} agents can achieve greater environmental reward for the same amount of penalty. To counterbalance this incentive, we multiply $R_\text{env}$ by $(1-\gamma)$.

We reuse the hyperparameters of~\cite{turner2020conservative}.{\footnote{Code available at \href{https://github.com/aseembits93/attainable-utility-preservation}{https://github.com/aseembits93/attainable-utility-preservation}.}} The learning rate is $\alpha\defeq 1$, and the discount rate is $\gamma\defeq .996$. We use the following {\aup} hyperparameter values: penalty coefficient $\lambda\defeq 0.01$, $\abs{\mathcal{R}}\defeq 20$ randomly generated auxiliary reward functions.

\subsection{Additional experiment}
We also tested an {\aup}-like agent optimizing reward function $R_\text{power penalty}(s,a)\defeq R_\text{env}(s,a) - \frac{\lambda}{\abs{\R}}\sum_{R_i \in \R} Q^*_{R_i}(s,a)-Q^*_{R_i}(s,\varnothing)$, which is the {\aup} reward function (\cref{def:aup-rf}) without the absolute value. This objective penalizes the agent for changes in its average optimal value, which is related to~\cite{turner_optimal_2020}'s $\pwr$. $R_\text{power penalty}$ produced the same prefix policies as $\rAUP$, and hence the same delayed specification scores.
\section{Theoretical results}
All results only apply to {\mdp}s with finite state and action spaces.

\begin{restatable}[Average optimal value~\cite{turner_optimal_2020}]{definition}{FORMavgVal}\label{FORM:def:vavg}
For bounded-support reward function distribution $\D$, the \emph{average optimal value} at state $s$ and discount rate $\gamma \in (0,1)$ is $\vavg[s,\gamma][\D]\defeq\E{R\sim \D}{\OptVf{s,\gamma}}.$
\end{restatable}

\uncertainOpt*
\begin{proof}
Suppose the agent starts at state $s_0$, and let $\gamma \in (0,1)$.
\begin{align}
   &\E{\substack{t \sim \timeDist,\\R\sim\D}}{\valSwitch{s_0,\gamma}} \\
   =& (1-\gamma) \E{\substack{t \sim \timeDist,\\R\sim\D}}{\sum_{i=0}^\infty \gamma^i \E{s_i \sim \piSwitch{\pi}}{R(s_i)}}\\
    =& (1-\gamma)\E{\substack{t \sim \timeDist,\\R\sim\D}}{\sum_{i=0}^{t-1} \gamma^i \E{s_i \sim \pi}{R(s_i)} + \gamma^t \E{s_t \sim \pi}{\Vf{s_t,\gamma}}}\label{eq:uncertain-piswitch-defn}\\
    =& (1-\gamma)\E{t \sim \timeDist}{\sum_{i=0}^{t-1} \gamma^i \E{s_i \sim \pi}{\ravg(s_i)}}\nonumber\\
    &\qquad+  \E{t \sim \timeDist,\, s_t \sim \pi}{\gamma^t (1-\gamma) \vavg[s_t,\gamma]}\label{eq:uncertain-lin-reward}\\
    =& (1-\gamma) \E{t \sim \timeDist}{\sum_{i=0}^{t} \gamma^i \E{s_i \sim \pi}{\ravg(s_i)}}\nonumber\\
    &\qquad + \E{t \sim \timeDist,\, s_t \sim \pi}{\gamma^{t+1} \pwr[s_t,\gamma]}.\label{eq:pwr-decomp}
\end{align}

\Cref{eq:uncertain-piswitch-defn} follows from the definition of the non-stationary $\piSwitch{\pi}$, and the fact that each $\pi^*_R$ is optimal for each $R$ at discount rate $\gamma$. \Cref{eq:uncertain-lin-reward} follows by the linearity of expectation and the definition of $\ravg$ and the definition of $\pwr$ (\cref{def:powRestate}). \Cref{eq:pwr-decomp} follows because $\vavg[s_t,\gamma]=\geom[\gamma] \pwr[s_t,\gamma] + \ravg(s_t)$.

If $\gamma = 0$ or $\gamma = 1$, the result holds in the respective limit because $\pwr$ has well-defined limits by Lemma 5.3 of~\cite{turner_optimal_2020}.
\end{proof}

\pwrMaxProp*
\begin{proof}
\Cref{item:1-special}: if  $\forall s_1, s_2 \in \St: \ravg(s_1)=\ravg(s_2)$, the first term on the right-hand side of \cref{eq:tradeoff-power-time} is constant for all policies $\pi$; if $\gamma=1$, this first term equals $0$. Therefore, under these conditions, $\pi$ maximizes \cref{eq:tradeoff-power-time} iff it maximizes $\E{t \sim \timeDist, s_t \sim \pi}{\gamma^{t+1}\pwr[s_t,\gamma]}$.

\Cref{item:2-special}: under these conditions, the second term on the right-hand side of \cref{eq:tradeoff-power-time} is constant for all policies $\pi$. Therefore, $\pi$ maximizes \cref{eq:tradeoff-power-time} iff it maximizes
\[(1-\gamma)\E{t \sim \timeDist}{\sum_{i=0}^{t} \gamma^i \E{s_i \sim \pi}{\ravg(s_i)}}.\]

If both \cref{item:1-special} and \cref{item:2-special} hold, it trivially follows that all $\pi$ are optimal prefix policies. If $\gamma=0$, all $\pi$ are optimal prefix policies, since reward is state-based and no actions taken by $\pi$ can affect expected return $\mathbf{ER}$.
\end{proof}

\timeDistGeom*
\begin{proof}
\begin{align}
    &\argsup_\pi \E{\substack{t \sim \timeDist,\\R\sim\D}}{\valSwitch{s,\gamma}}\label{eq:aup-soln-argmax}\\
    =&\argsup_\pi \E{t \sim \timeDist}{\sum_{i=0}^{t-1} \gamma^i \E{s_i \sim \pi \mid s_0}{\ravg(s_i)}} +   \E{\substack{t \sim \timeDist,\\s_t \sim \pi \mid s_0}}{\gamma^{t}\vavg[s_t,\gamma]}\label{eq:aup-soln-decomp}\\
    =&\argsup_\pi \sum_{t=1}^\infty \prob{\timeDist=t} \sum_{i=0}^{t-1} \gamma^i \E{s_i \sim \pi \mid s_0}{\ravg(s_i)} +   \E{\substack{t \sim \timeDist,\\s_t \sim \pi \mid s_0}}{\gamma^{t}\vavg[s_t,\gamma]}\\
    =&\argsup_\pi \sum_{t=1}^\infty (1-p)^{t-1}p \sum_{i=0}^{t-1} \gamma^i \E{s_i \sim \pi \mid s_0}{\ravg(s_i)} +   \sum_{t=1}^\infty (1-p)^{t-1}p\E{s_t \sim \pi \mid s_0}{\gamma^{t}\vavg[s_t,\gamma]}\label{eq:aup-soln-geom}\\
     =&\argsup_\pi \sum_{t=1}^\infty (1-p)^{t-1} p \sum_{i=0}^{t-1} \gamma^i \E{s_i \sim \pi \mid s_0}{\ravg(s_i)} + p\gamma \sum_{t=1}^\infty \gamAUP^{t-1}\E{s_t \sim \pi \mid s_0}{\vavg[s_t,\gamma]}\label{eq:aup-soln-gam1}\\
     =&\argsup_\pi \sum_{i=0}^\infty \gamAUP^i \E{s_i \sim \pi \mid s_0}{\ravg(s_i)} + p\gamma \sum_{i=1}^\infty \gamAUP^{i-1}\E{s_i \sim \pi \mid s_0}{\vavg[s_i,\gamma]}\label{eq:aup-soln-gam2}\\
     =&\argsup_\pi \ravg(s_0)+ \sum_{i=1}^\infty \gamAUP^{i-1} \E{s_i \sim \pi \mid s_0}{(1-p)\gamma\ravg(s_i) + p\gamma  \vavg[s_i,\gamma]}\label{eq:aup-soln-refactor}\\
     =&\argsup_\pi \sum_{i=1}^\infty \gamAUP^{i} \E{s_i \sim \pi \mid s_0}{(1-p)\ravg(s_i) + p\vavg[s_i,\gamma]}\label{eq:aup-soln-remove}\\
      =&\argsup_\pi \sum_{i=0}^\infty \gamAUP^{i} \E{s_i \sim \pi \mid s_0}{(1-p)\ravg(s_i) + p\vavg[s_i,\gamma]}\label{eq:aup-soln-add}\\
     =& \argsup_\pi V^\pi_{R'}(s_0,\gamAUP).\label{eq:argmax-Vpi}
\end{align}

\Cref{eq:aup-soln-argmax} follows from \cref{thm:uncertain-opt}. \Cref{eq:aup-soln-geom} follows because $\timeDist=\geomDist$. \Cref{eq:aup-soln-gam1} follows by the definition of $\gamAUP$.

In \cref{eq:aup-soln-gam1}, consider the double-sum on the left. For any given $i$, the portion of the sum with factor $\gamma^i$ equals
\begin{align}
    &\gamma^i\E{s_i \sim \pi \mid s_0}{\ravg(s_i)}\sum_{j=i}^\infty(1-p)^j p\nonumber\\
    &\quad=(1-p)^i\gamma^i\E{s_i \sim \pi \mid s_0}{\ravg(s_i)}p\sum_{j=0}^\infty(1-p)^{j}\\
   &\quad= (1-p)^i\gamma^i\E{s_i \sim \pi \mid s_0}{\ravg(s_i)}p\frac{1}{1-(1-p)}\label{eq:finite}\\
   &\quad= \gamAUP^i\E{s_i \sim \pi \mid s_0}{\ravg(s_i)}\frac{p}{1-(1-p)}\\
   &\quad= \gamAUP^i\E{s_i \sim \pi \mid s_0}{\ravg(s_i)}.\label{eq:i-factor}
\end{align}

The geometric identity holds for \cref{eq:finite} because $p>0 \implies (1-p)<1$. Therefore, \cref{eq:aup-soln-gam2} follows from \cref{eq:i-factor}.

\Cref{eq:aup-soln-refactor} follows by extracting the leading constant of $\ravg(s_0)$ and then expanding one of the $\gamAUP\defeq (1-p)\gamma$ factors of the first series. \Cref{eq:aup-soln-remove} follows by subtracting the constant $\ravg(s_0)$ by multiplying by $(1-p)> 0$, and by the fact that $(1-p)\gamma=\gamAUP$. \Cref{eq:aup-soln-add} follows because adding the constant $(1-p)\ravg(s_0)+p\vavg[s_0,\gamma]$ does not change the $\argsup$.

\Cref{eq:argmax-Vpi} follows by the definition of an on-policy value function and by the definition of $R'$. But $s_0$ was arbitrary, and so this holds for every state. Then  the policies in $\optPi[R',\gamAUP]$ satisfy the $\argsup$ for all states. $\optPi[R',\gamAUP]$ is non-empty because the {\mdp} is finite.
\end{proof}

\aupSolves*
\begin{proof}

\begin{align}
    &\argmax_\pi \E{\substack{t \sim \timeDist,\\R\sim\D}}{\valSwitch{s,\gamma}}\\
    =&\argmax_\pi \sum_{i=0}^\infty \gamAUP^{i} \E{s_i \sim \pi \mid s_0}{(1-p)\ravg(s_i) + p\vavg[s_i,\gamma]}\label{eq:aup-soln-add-repeat}\\
    =&\argmax_\pi \sum_{i=0}^\infty \gamAUP^{i} \E{s_i \sim \pi \mid s_0}{(1-p)\ravg(s_i) + p\vavg[s_i,\gamma]} \!-\! p\! \sum_{i=0}^\infty \gamAUP^{i} \!\E{s_i^\varnothing \sim \baseline \mid s_0}{\vavg[s_i^\varnothing,\gamma]}\label{eq:aup-soln-minus-inaction}\\
    =&\argmax_\pi \sum_{i=0}^\infty \gamAUP^{i} \E{s_i \sim \pi \mid s_0}{(1-p)\ravg(s_i) - p\prn{\E{s_i^\varnothing \sim \baseline \mid s_0}{\vavg[s_i^\varnothing,\gamma]}-\vavg[s_i,\gamma]}}   \\
    =&\argmax_\pi \sum_{i=0}^\infty \gamAUP^{i} \E{s_i \sim \pi \mid s_0}{(1-p)\ravg(s_i) - p\E{R\sim \D}{\E{s_i^\varnothing \sim \baseline \mid s_0}{\Vf{s_i^\varnothing,\gamma}}-\Vf{s_i,\gamma}}}\label{eq:aup-soln-expect}    \\
    =&\argmax_\pi \sum_{i=0}^\infty \gamAUP^{i} \E{s_i \sim \pi \mid s_0}{\ravg(s_i) - \frac{p}{1-p}\E{R\sim \D}{\E{s_i^\varnothing \sim \baseline \mid s_0}{\Vf{s_i^\varnothing,\gamma}}-\Vf{s_i,\gamma}}}.\label{eq:aup-soln-final}
\end{align}

\Cref{eq:aup-soln-add-repeat} follows from \cref{thm:time-geom-stationary}. \Cref{eq:aup-soln-minus-inaction} only subtracts a constant. \Cref{eq:aup-soln-expect} follows from the definition of $\vavg$, and \cref{eq:aup-soln-final} follows because dividing by $(1-p)>0$ does not affect the $\argmax$. But the expectation of \cref{eq:aup-soln-final} takes an expectation over $\rAssist(s_i \mid s)$, and its $\argmax$ equals the set of optimal policies for $\rAssist$ starting from state $s_0$.
\end{proof}

\end{document}